\documentclass[conference]{IEEEtran}
\IEEEoverridecommandlockouts

\usepackage{cite}
\usepackage{amsmath,amssymb,amsfonts}
\usepackage{algorithmic}
\usepackage{graphicx}
\usepackage{textcomp}
\usepackage{xcolor}

\usepackage{amsthm}

\theoremstyle{plain}
\newtheorem{theorem}{Theorem}
\newtheorem{corollary}[theorem]{Corollary}
\newtheorem{lemma}[theorem]{Lemma}
\newtheorem{proposition}[theorem]{Proposition}
\theoremstyle{definition}
\newtheorem{definition}[theorem]{Definition}
\newtheorem{example}[theorem]{Example}
\theoremstyle{remark}
\newtheorem{remark}[theorem]{Remark}

\usepackage{booktabs}
\usepackage{array}

\newcommand{\EMD}{\mathrm{EMD}}
\newcommand{\ECR}{\mathrm{ECR}}
\newcommand{\ACR}{\mathrm{ACR}}
\DeclareMathOperator{\tr}{tr}
\DeclareMathOperator{\ran}{ran}
\newcommand{\symdiff}{\,\triangle\,}
\newcommand{\Ep}{E_{\mathrm{p}}}
\newcommand{\Ew}{E_{\mathrm{w}}}
\newcommand{\pin}{p_{\mathrm{in}}}
\newcommand{\pout}{p_{\mathrm{out}}}

\def\BibTeX{{\rm B\kern-.05em{\sc i\kern-.025em b}\kern-.08em
    T\kern-.1667em\lower.7ex\hbox{E}\kern-.125emX}}
\begin{document}

\title{A Spectral Theory of Distortion in LLM Graph Reconstruction:
Sharp Bounds and Empirical Characterization}

\author{\IEEEauthorblockN{Jianru Shen}
\IEEEauthorblockA{\textit{Columbia University} \\
New York, NY, USA \\
js5929@columbia.edu}}

\maketitle

\begin{figure*}[t]
\centering
\includegraphics[width=\textwidth]{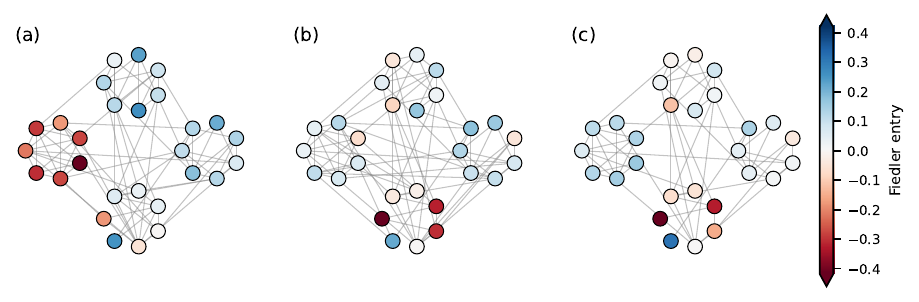}
\caption{Two reconstructions of community instance g5 ($n=28$); layout
and node color are fixed from the original graph, so only the edge
sets differ. (a)~Original. (b)~Preserves the edge count ($\ECR=1$)
while adding and deleting $19$ edges each, certified mixed from scalar
summaries. (c)~Verbatim return of the prompted sublist. EMD and
edge-count summaries can thus obscure the edit mechanism; neither
reconstruction is implied to be faithful.}
\label{fig:teaser}
\end{figure*}

\begin{abstract}
Evaluations of graph reconstruction by language models typically
report a single aggregate distance between the original and the
reconstructed graph. We prove that for the Wasserstein distance
between Laplacian spectra such a summary is bracketed by two edge
counts, the net change in edge number from below and the symmetric
difference from above, each scaled by $2/n$ where $n$ is the number of
vertices. The bracket is sharp: its two ends coincide
exactly when the reconstruction only adds edges or only deletes them,
and on that class the distance is a rescaled edge count that says
nothing about which edges changed. When the ends differ, the residual
between the distance and the lower end is positive only if the
reconstruction both invented and lost edges, which turns it into a
certificate of mixed editing computable from the reported summaries
alone. We characterize these regimes in 135 reconstructions produced by
three open-weight models over 45 synthetic graphs. Seventy-seven
outputs are one-sided and 29 mixed outputs have $X>0$, including cases
where edge count is exactly preserved while nineteen edges were
simultaneously invented and lost. The three models differ in editing policy, ranging from copying the
input to attempting completion at the cost of large hallucination
volume, a distinction that aggregate distortion does not reveal.
\end{abstract}


\section{Introduction}
\label{sec:intro}

Large language models are increasingly applied to graph-structured
tasks, including node classification, link prediction, and reasoning
over topology described in natural language
\cite{Jin:2023,Wang:2023,Guo:2023}. A basic diagnostic in this setting
is reconstruction: the model receives a partial description of a graph
and returns what it takes the complete edge set to be. Assessing the
result means comparing two graphs on a common vertex set, and the
comparison is usually reduced to a scalar, such as edge accuracy or
edit distance \cite{Wang:2023,Guo:2023,Fatemi:2023}, or a
Laplacian-spectral distance \cite{GuHuaLiu:2015,Tsitsulin:2018}.

One scalar cannot say how a reconstruction failed. A model that omits
twenty edges and a model that invents twenty edges can produce the
same distance to the original, although the first is conservative and
the second hallucinates. This is usually treated as a known
imprecision of aggregate metrics. We show that for spectral distortion
it is a theorem, and that the same argument yields a test that repairs
part of the loss at no additional cost. Figure~\ref{fig:teaser} shows two reconstructions of one graph whose
scalar summaries obscure different edit mechanisms.

Let $G'$ be a reconstruction of $G$ on the same vertex set and let
$\EMD$ denote the Wasserstein distance between their Laplacian
spectra. Our main result brackets $\EMD$ between two edge counts: the
net change $\bigl||E|-|E'|\bigr|$ from below and the symmetric
difference $|E\symdiff E'|$ from above, both scaled by $2/n$. The
bracket is sharp in a strong sense. Its two ends coincide exactly when
the edit is one-sided, that is when the model only added edges or only
deleted them, and on that class $\EMD$ equals a rescaled edge count
and carries no further information. When the two ends differ, the
residual between $\EMD$ and the lower end is positive only if the edit
is mixed, which turns that residual into a certificate of
simultaneous hallucination and loss. Once EMD, ECR, $|E|$ and $n$ have
been computed, the certificate requires no access to the identities of
the reconstructed edges.

We evaluate three open-weight models on 45 synthetic graphs spanning
block, preferential attachment, and lattice families. The bounds hold
on all 135 reconstructions. Seventy-seven fall in the degenerate class
where distortion is provably an edge count, so a per family average of
$\EMD$ over such a corpus reproduces the density profile of the
benchmark rather than any property of a model. Twenty-nine mixed outputs have $X>0$, including two
edge-count-preserving outputs for which no individual scalar reveals
the simultaneous additions and deletions of nineteen and fifteen edges
respectively. Aggregate distortion separates the three models by magnitude but
obscures substantial differences in editing policy, ranging from
returning the prompt verbatim to attempting completion at the cost of
large hallucination volume.

Prior work on graph reasoning with language models evaluates output at
the task level or through edit distance
\cite{Wang:2023,Fatemi:2023,Guo:2023} and treats the metric as given.
For graph comparison, Gu, Hua, and Liu define Wasserstein distances
between normalized-Laplacian spectral measures and derive $O(1/n)$
stability under bounded graph operations \cite{GuHuaLiu:2015}, while
NetLSD compares graphs through Laplacian heat-trace signatures
\cite{Tsitsulin:2018}. Our setting instead fixes a common vertex set
and the combinatorial Laplacian and yields exact finite-$n$ edge-edit
bounds in terms of the net edge-count change and the symmetric
difference, together with equality characterizations and a mixed-edit
certificate. The perturbation inequalities used in the proof, due to
Weyl and Mirsky \cite{Mirsky:1960,Bhatia:1997}, are classical; the
contribution is the resulting characterization and its empirical
regime analysis.

\section{Spectral Bounds for Reconstruction Distortion}
\label{sec:theory}

\subsection{Notation}
\label{sec:notation}

Let $G=(V,E)$ be a simple undirected graph and let $G'=(V,E')$ be a
reconstruction of $G$ on the same vertex set, $|V|=n$. Their
combinatorial Laplacians are $L=D-A$ and $L'=D'-A'$, with degree
matrices $D,D'$ and adjacency matrices $A,A'$. Eigenvalues are always
listed in nondecreasing order,
$\lambda_1\le\cdots\le\lambda_n$ for $L$ and
$\lambda'_1\le\cdots\le\lambda'_n$ for $L'$, and we write
$\delta_i=\lambda'_i-\lambda_i$ for the pointwise spectral difference
at index $i$. The Fiedler value is $\lambda_2$ and the Fiedler vector
$v$ is a unit eigenvector for it.

An edit is described by the added set $E'\setminus E$ and the deleted
set $E\setminus E'$, with cardinalities $a=|E'\setminus E|$ and
$d=|E\setminus E'|$, so that $|E\symdiff E'|=a+d$ and
$|E'|-|E|=a-d$. An edit is \emph{one-sided} if $\min(a,d)=0$ and
\emph{mixed} otherwise. For an edge $e=\{u,v\}$ let $b_e$ be its
signed incidence vector, with entries $+1$ at $u$, $-1$ at $v$, and
$0$ elsewhere; then $b_eb_e^\top\succeq0$ and
$\tr(b_eb_e^\top)=\|b_e\|^2=2$. The perturbation induced by the edit
is
\begin{equation}
\Delta=L'-L=P-M,\quad
P=\!\!\!\sum_{e\in E'\setminus E}\!\!\! b_eb_e^\top,\quad
M=\!\!\!\sum_{e\in E\setminus E'}\!\!\! b_eb_e^\top,
\label{eq:delta}
\end{equation}
with $P,M\succeq0$, and $\|\cdot\|_*$ denotes the trace norm.

We use three scalar summaries. The spectral distortion is the
Wasserstein distance between the empirical spectral measures
$\mu_G=\frac1n\sum_{i=1}^n\delta_{\lambda_i}$ and
$\mu_{G'}=\frac1n\sum_{i=1}^n\delta_{\lambda'_i}$, which on the real
line is realized by the monotone coupling \cite{Santambrogio:2015} and
therefore equals
\begin{equation}
\EMD(G,G')=W_1(\mu_G,\mu_{G'})
=\frac1n\sum_{i=1}^{n}\bigl|\delta_i\bigr|.
\label{eq:emd}
\end{equation}
The edge count ratio is $\ECR=|E'|/|E|$, defined whenever $|E|>0$ as
holds for all benchmark graphs, and the algebraic
connectivity ratio is $\ACR=\lambda'_2/\lambda_2$, defined when
$\lambda_2>0$. Section~\ref{sec:setup} presents each graph to a model
as a sublist $\Ep\subseteq E$ of size $\lfloor\rho|E|\rfloor$ for a
fixed keep ratio $\rho$, together with the withheld set
$\Ew=E\setminus\Ep$, and asks the model to return the complete edge
list. The set $E'$ is the model output. Community graphs are drawn
from a stochastic block model \cite{Holland:1983} with $c$ blocks of
equal size $s=n/c$, within block probability $\pin$ and between block
probability $\pout$; scale-free graphs use the preferential attachment
model with attachment parameter $m$.

\subsection{The sandwich bound}
\label{sec:sandwich}

\begin{theorem}[Spectral sandwich]
\label{thm:sandwich}
For all $G,G'$ on a common vertex set of size $n$,
\begin{equation}
\frac{2}{n}\bigl||E|-|E'|\bigr|
\;\le\;\EMD(G,G')\;\le\;
\frac{2}{n}\bigl|E\symdiff E'\bigr|.
\label{eq:sandwich}
\end{equation}
The lower bound equals $\frac{2|E|}{n}\bigl|1-\ECR\bigr|$, so it is
computable from the scalar summaries alone.
\end{theorem}

\begin{proof}
For the lower bound, $\tr L=\sum_{u\in V}\deg(u)=2|E|$ and likewise
$\tr L'=2|E'|$, so
\begin{equation*}
\sum_{i}|\delta_i|\;\ge\;\Bigl|\sum_i\delta_i\Bigr|
=\bigl|\tr L'-\tr L\bigr|=2\bigl||E'|-|E|\bigr|.
\end{equation*}
For the upper bound, Mirsky's inequality for the trace norm
\cite{Mirsky:1960,Bhatia:1997} gives
$\sum_i|\delta_i|\le\|\Delta\|_*$. Applying the triangle inequality
for $\|\cdot\|_*$ to \eqref{eq:delta} and using
$\|b_eb_e^\top\|_*=\tr(b_eb_e^\top)=2$ yields
$\|\Delta\|_*\le2|E\symdiff E'|$. Dividing by $n$ and rewriting
$\bigl||E|-|E'|\bigr|=|E|\,|1-\ECR|$ completes the proof.
\end{proof}

\begin{remark}
\label{rem:lower}
Equality holds on the left of \eqref{eq:sandwich} if and only if the
nonzero differences $\delta_i$ all carry the same sign, that is, if
and only if one spectrum dominates the other pointwise. Pointwise
dominance follows from a one-sided edit by Weyl monotonicity, but the
converse fails; Example~\ref{ex:k4} exhibits a mixed edit whose
perturbation is positive semidefinite.
\end{remark}

\begin{corollary}[One-sided collapse]
\label{cor:collapse}
If $E\subseteq E'$ or $E'\subseteq E$, then
\begin{equation}
\EMD(G,G')=\frac{2}{n}\bigl|E\symdiff E'\bigr|
=\frac{2}{n}\bigl||E|-|E'|\bigr|,
\label{eq:collapse}
\end{equation}
so both bounds of \eqref{eq:sandwich} are attained at once. The two
bounds coincide if and only if the edit is one-sided, since
$|a-d|=a+d$ exactly when $\min(a,d)=0$. In this regime $\EMD$ is a
rescaled edge count and carries no information about the edit beyond
its cardinality.
\end{corollary}

\begin{proof}
For a pure addition $\Delta=P\succeq0$, so Weyl monotonicity gives
$\delta_i\ge0$ for every $i$ and
$\sum_i|\delta_i|=\sum_i\delta_i=\tr\Delta=2a$, which equals both
$2|E\symdiff E'|$ and $2\bigl||E'|-|E|\bigr|$. Pure deletion is
symmetric. The second claim is the stated arithmetic identity.
\end{proof}

Corollary~\ref{cor:collapse} shows that the sandwich is sharp: its two
bounds meet on an entire class of edits, and that class is exactly the
one-sided one. It also identifies the regime in which $\EMD$ is
uninformative about mechanism. The next result converts the residual
between $\EMD$ and the lower bound into a certificate for the
complementary regime.

\begin{corollary}[Mixing certificate]
\label{cor:cert}
Define the trace excess
\begin{equation}
X=\EMD(G,G')-\frac{2|E|}{n}\bigl|1-\ECR\bigr|\;\ge\;0.
\label{eq:excess}
\end{equation}
Then
\begin{equation}
\min(a,d)\;\ge\;\frac{n}{4}\,X .
\label{eq:cert}
\end{equation}
In particular $X>0$ proves that the reconstruction both hallucinated
and lost edges. Once $\EMD$, $\ECR$, $|E|$ and $n$ are available,
computing $X$ requires no access to the identities of the
reconstructed edges.
\end{corollary}

\begin{proof}
By Theorem~\ref{thm:sandwich}, $X$ is at most the width of the
interval in \eqref{eq:sandwich}, namely
$\frac{2}{n}\bigl[(a+d)-|a-d|\bigr]=\frac{4}{n}\min(a,d)$.
Equivalently, writing $S_+=\sum_{\delta_i>0}\delta_i$ and
$S_-=-\sum_{\delta_i<0}\delta_i$, we have $\sum_i|\delta_i|=S_++S_-$
and $\bigl|\sum_i\delta_i\bigr|=|S_+-S_-|$, so
$nX=S_++S_--|S_+-S_-|=2\min(S_+,S_-)$. Thus $X$ is the normalized
amount of spectral motion that is cancelled by motion in the opposite
direction.
\end{proof}

\begin{definition}[Edit classes]
\label{def:classes}
A reconstruction is \emph{one-sided} if $\min(a,d)=0$,
\emph{dominant mixed} if $\min(a,d)\ge1$ and $X=0$, and
\emph{certified mixed} if $X>0$. The three classes partition all
reconstructions, and membership in the third is decidable from the
scalar summaries alone.
\end{definition}

\begin{remark}[Conservativeness]
\label{rem:tight}
Equality in \eqref{eq:cert} forces $\EMD$ to attain the upper bound of
\eqref{eq:sandwich}, which by Proposition~\ref{prop:upper} requires
every added edge to be vertex disjoint from every deleted edge. The
certificate is therefore expected to be conservative whenever added
and deleted edges overlap. It lower-bounds the smaller edit side but
need not recover its true size.
\end{remark}

\subsection{Attainment of the upper bound}
\label{sec:upper}

\begin{lemma}[Splitting]
\label{lem:split}
Let $P,M\succeq0$ be symmetric. Then $\|P-M\|_*=\tr P+\tr M$ if and
only if $PM=0$.
\end{lemma}

\begin{proof}
If $PM=0$ then $\ran M\subseteq\ker P$ and $\ran P\subseteq\ker M$, so
$P-M$ acts as $P$ and as $-M$ on orthogonal subspaces and its trace
norm is $\tr P+\tr M$.

Conversely, write $\Delta=P-M$ in its Jordan decomposition
$\Delta=\Delta_+-\Delta_-$ with $\Delta_\pm\succeq0$ supported on
orthogonal subspaces, so $\|\Delta\|_*=\tr\Delta_++\tr\Delta_-$.
Together with $\tr\Delta_+-\tr\Delta_-=\tr\Delta=\tr P-\tr M$, the
hypothesis gives $\tr\Delta_+=\tr P$ and $\tr\Delta_-=\tr M$. Let
$\Pi$ be the orthogonal projection onto $\ran\Delta_+$. Orthogonality
of the supports gives $\Delta_-\Pi=0$ and hence
$\Pi\Delta\Pi=\Delta_+$, so
\begin{align*}
\tr P=\tr\Delta_+
&=\tr(\Pi P\Pi)-\tr(\Pi M\Pi)\\
&=\tr P-\tr\bigl((I{-}\Pi)P(I{-}\Pi)\bigr)-\tr(\Pi M\Pi).
\end{align*}
Both subtracted terms are traces of positive semidefinite matrices and
must vanish. From $\tr(\Pi M\Pi)=\|M^{1/2}\Pi\|_F^2=0$ we get
$M\Pi=0$, and from
$\tr\bigl((I{-}\Pi)P(I{-}\Pi)\bigr)=\|P^{1/2}(I{-}\Pi)\|_F^2=0$ we get
$P=\Pi P\Pi$. Hence $P$ is supported in $\ran\Pi$ while $M$
annihilates it, so $PM=0$.
\end{proof}

\begin{proposition}[Necessary condition for upper attainment]
\label{prop:upper}
Suppose $\EMD(G,G')=\frac{2}{n}|E\symdiff E'|$ and the edit is mixed.
Then every added edge is vertex disjoint from every deleted edge.
\end{proposition}

\begin{proof}
Attainment forces equality throughout
$\sum_i|\delta_i|\le\|\Delta\|_*\le2(a+d)$, in particular
$\|P-M\|_*=\tr P+\tr M$, so $PM=0$ by Lemma~\ref{lem:split}.
Since $\tr(PM)=\|P^{1/2}M^{1/2}\|_F^2$, this is equivalent to
$\tr(PM)=0$, and expanding \eqref{eq:delta} gives
\begin{equation*}
\tr(PM)=\sum_{e\in E'\setminus E}\ \sum_{f\in E\setminus E'}
\bigl(b_e^\top b_f\bigr)^2 .
\end{equation*}
The edges $e$ and $f$ are always distinct, so
$b_e^\top b_f\in\{0,\pm1\}$, and it is nonzero exactly when $e$ and
$f$ share a vertex. The sum vanishes if and only if no added edge
meets a deleted edge.
\end{proof}

\begin{figure}[t]
\centering
\includegraphics[width=\columnwidth]{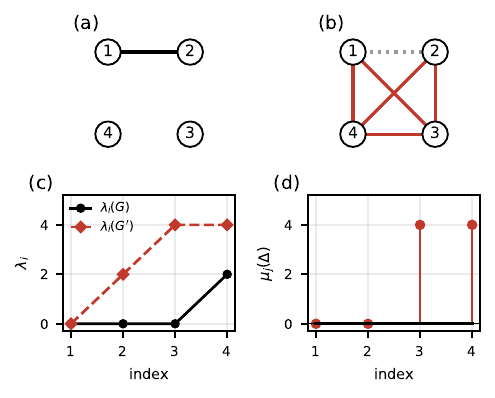}
\caption{The mixed edit of Example~\ref{ex:k4}. (a)~$G$. (b)~$G'=K_4$
minus $\{1,2\}$; hallucinated edges solid red, deleted edge dotted
gray. (c)~Pointwise dominance without a crossing, attaining the lower
bound of \eqref{eq:sandwich}. (d)~The perturbation spectrum is
nonnegative although the edit is mixed.}
\label{fig:k4}
\end{figure}

\begin{example}[Mixed edit attaining the lower bound]
\label{ex:k4}
Let $n=4$, $E=\{\{1,2\}\}$ and $E'=E(K_4)\setminus\{\{1,2\}\}$, so
$a=5$ and $d=1$. The spectra are $\{0,0,0,2\}$ and $\{0,2,4,4\}$,
hence $\EMD=2$, the lower bound $\frac24|5-1|=2$ is attained and the
upper bound $\frac24\cdot6=3$ is strict. Here
$\Delta=L(K_4)-2\,b_{12}b_{12}^\top$ has spectrum $\{0,0,4,4\}$: a
mixed edit can produce a positive semidefinite perturbation and hence
pointwise spectral dominance. This is why the certificate of
Corollary~\ref{cor:cert} is one directional. Consistently with
Proposition~\ref{prop:upper}, the upper bound is not attained, since
the added edge $\{1,3\}$ meets the deleted edge $\{1,2\}$ at vertex
$1$. Figure~\ref{fig:k4} displays the construction.
\end{example}

\begin{example}[Isospectral edit]
\label{ex:iso}
Let $n=4$, $E=\{\{1,2\},\{2,3\},\{1,4\}\}$ and
$E'=\{\{2,3\},\{3,4\},\{1,4\}\}$, so $a=d=1$ and the added edge
$\{3,4\}$ is vertex disjoint from the deleted edge $\{1,2\}$. Both
graphs are paths on four vertices, so their spectra coincide,
$G$ is connected with $\lambda_2=2-\sqrt2>0$, and every summary is
blind: $\EMD=0$ while $|E\symdiff E'|=2$, and $X=0$, $\ECR=1$,
$\ACR=1$, with no crossing in the sense of
Proposition~\ref{prop:crossing}. Together with Example~\ref{ex:k4}
this brackets the failure modes of spectral diagnostics, a mixed edit
presenting either as dominance or as the identity. It also shows that
the vertex disjointness of Proposition~\ref{prop:upper} is necessary
but not sufficient for upper attainment, since this disjoint edit
attains the lower bound instead.
\end{example}

\begin{example}[Upper attainment and a tight certificate]
\label{ex:star}
Let $G$ be the star $K_{1,3}$ with center $1$, and let the
reconstruction delete the leaf edge $\{1,2\}$ and add the edge
$\{3,4\}$ between two leaves, so $a=d=1$ and $|E'|=|E|$. The spectra
are $\{0,1,1,4\}$ and $\{0,0,3,3\}$, giving $\EMD=1$, which attains
the upper bound $\frac{2}{4}\cdot2=1$. The added and deleted edges
are vertex disjoint, as required by Proposition~\ref{prop:upper}. Here the net change is zero, so the
lower bound is $0$ and the trace excess is $X=1$; the certificate of
Corollary~\ref{cor:cert} returns $\lceil nX/4\rceil=1=\min(a,d)$ and
is exact. This shows that the constant $\tfrac14$ in \eqref{eq:cert}
cannot be improved.
\end{example}

\subsection{Spectral order and local diagnostics}
\label{sec:order}

\begin{proposition}[Order and crossings]
\label{prop:crossing}
Let $\delta_i=\lambda'_i-\lambda_i$.
\begin{enumerate}
\item[(i)] If $E\subseteq E'$ then $\delta_i\ge0$ for all $i$; if
$E'\subseteq E$ then $\delta_i\le0$ for all $i$.
\item[(ii)] If $E'=E\cup\{e\}$ then
$\lambda_i\le\lambda'_i\le\lambda_{i+1}$ for $i<n$,
$\lambda_n\le\lambda'_n\le\lambda_n+2$, and $\sum_i\delta_i=2$.
\item[(iii)] $X>0$ if and only if there exist indices $i,j$ with
$\delta_i>0>\delta_j$.
\end{enumerate}
Consequently a crossing of the two sorted spectral curves is a sound
certificate of a mixed edit, equivalent to $X>0$; the converse
implication from mixing to crossing fails by Example~\ref{ex:k4}.
\end{proposition}

\begin{proof}
Part (i) is Weyl monotonicity for $\Delta=P\succeq0$, respectively
$\Delta=-M\preceq0$. Part (ii) is the Weyl interlacing inequality for
a rank one positive semidefinite update of trace $2$. For (iii), the
identity $nX=2\min(S_+,S_-)$ from Corollary~\ref{cor:cert} shows that
$X>0$ if and only if both positive and negative spectral shifts occur.
A one-sided edit has only one sign by (i).
\end{proof}

\begin{proposition}[Connectivity diagnostics]
\label{prop:acr}
Let $G$ be connected, so $\lambda_2>0$.
\begin{enumerate}
\item[(i)] If $E\subseteq E'$ then $\ACR\ge1$, and if $E'\subseteq E$
then $\ACR\le1$. Hence $\ACR<1$ certifies $d\ge1$ and $\ACR>1$
certifies $a\ge1$.
\item[(ii)] If $\ECR\ge1$ and $\ACR<1$, or $\ECR\le1$ and $\ACR>1$,
then the edit is mixed and $X>0$. The pair $(\ECR,\ACR)$ is thus a
coarsening of Corollary~\ref{cor:cert} that localizes a crossing at
the Fiedler index.
\item[(iii)] For every $S$ with $\emptyset\ne S\subsetneq V$,
\begin{equation*}
\lambda'_2\le\frac{n\,|\partial_{G'}S|}{|S|\,(n-|S|)},
\end{equation*}
where $\partial_{G'}S$ is the set of edges of $G'$ leaving $S$. In
particular $\ACR=0$ if and only if $G'$ is disconnected.
\end{enumerate}
\end{proposition}

\begin{proof}
Part (i) is Proposition~\ref{prop:crossing}(i) at index $2$. For (ii),
assume $\ECR\ge1$ and $\ACR<1$. Then $\sum_i\delta_i=2(|E'|-|E|)\ge0$
while $\delta_2<0$. If every $\delta_i$ were nonpositive the sum would
force $\delta_i=0$ for all $i$, contradicting $\delta_2<0$; hence some
$\delta_j>0$ and Proposition~\ref{prop:crossing}(iii) applies. The
other case is symmetric. For (iii), take
$x=(n-|S|)\mathbf1_S-|S|\mathbf1_{V\setminus S}$. Then
$x\perp\mathbf1$, $x^\top x=n|S|(n-|S|)$, and only cut edges
contribute to $x^\top L'x=\sum_{\{u,v\}\in E'}(x_u-x_v)^2$, each by
$n^2$, so $x^\top L'x=n^2|\partial_{G'}S|$. The Rayleigh
characterization of $\lambda'_2$ gives the bound, and $\lambda'_2=0$
holds exactly when $G'$ is disconnected \cite{Fiedler:1973}.
\end{proof}

\subsection{The degenerate regime}
\label{sec:degenerate}

\begin{proposition}[Closed form under sublist return]
\label{prop:echo}
Suppose the reconstruction returns exactly the prompted sublist,
$E'=\Ep$ with $|\Ep|=\lfloor\rho|E|\rfloor$. Then the edit is one
sided and
\begin{equation}
\EMD(G,G')=\frac{2}{n}\bigl\lceil(1-\rho)|E|\bigr\rceil
=(1-\rho)\,\bar\kappa+O(1/n),
\label{eq:echoform}
\end{equation}
where $\bar\kappa=2|E|/n$ is the average degree of $G$. The value is
determined by $n$ and $|E|$ alone: it does not depend on the model, on
which edges were withheld, or on any structural property of $G$ beyond
its size and density.
\end{proposition}

\begin{proof}
Apply Corollary~\ref{cor:collapse} with
$|E|-\lfloor\rho|E|\rfloor=\lceil(1-\rho)|E|\rceil$.
\end{proof}

Proposition~\ref{prop:echo} has a direct consequence for benchmark
design. Along connected sparse families with average degree
$\bar\kappa=\Theta(1)$, including rectangular lattices and fixed-$m$
preferential-attachment graphs, the degenerate distortion is
$\Theta(1)$; for block models with fixed $\pin$ and growing block
size it grows with $\bar\kappa$. A
per family profile of $\EMD$ values can therefore reproduce the
density profile of the benchmark rather than any property of the
model under evaluation.

\section{Experimental Setup}
\label{sec:setup}

\subsection{Benchmark}

The corpus consists of 45 synthetic graphs in three families and
three size levels, five instances per cell. Synthetic graphs give
controlled generating families and exact ground truth for all
edge-level quantities; they do not make the withheld set uniquely
identifiable. Community graphs use an equal-block stochastic block
model \cite{Holland:1983} with $\pin=0.7$, $\pout=0.05$ and
$c=3,4,5$, giving $n=15,28,50$; two disconnected draws were
deterministically redrawn, so $\ACR$ is defined throughout. Scale
free graphs use preferential attachment \cite{Barabasi:1999} with
$m=2$ at $n=15,30,50$, giving $|E|=2(n-2)$. Grid graphs use five
pairwise non-isomorphic rectangular lattices per level, with $n$ in
$14$--$18$, $27$--$32$ and $45$--$50$.

\subsection{Reconstruction task}

Each graph is presented as a shuffled sublist $\Ep\subseteq E$ of size
$\lfloor\rho|E|\rfloor$ with keep ratio $\rho=0.75$, together with
three global descriptors: the number of vertices, the edge density and
the average degree. The model is instructed to return the complete
edge list, one edge per line, with no explanation. The sampling of
$\Ep$ uses a fixed seed, so the withheld set $\Ew=E\setminus\Ep$ is
identical across models and the completion problem is the same for
all of them.

We evaluate three instruction-tuned open-weight models at comparable
scale: Llama 3.1 8B \cite{Llama3:2024}, Mistral 7B
\cite{Mistral7B:2023} and Qwen 2.5 7B \cite{Qwen25:2024}. All are run
locally through Ollama \cite{Ollama} on a single Apple M4 machine with
temperature $0$ and no external API access, giving 135
reconstructions. Model output is parsed by fixing the vertex set to
$V$ and retaining valid undirected pairs of existing vertices, with
duplicates and self-loops removed. No other repair is applied:
deviations in edge count are the object of study rather than an
artifact to be corrected. Raw responses and parsed edge lists are
retained for every run, which permits all derived quantities to be
recomputed without re-querying the models.

Edge counts are reported after de-duplication. Across the 135
reconstructions the parser discarded 174, 334 and 1 candidate pairs
for Llama, Mistral and Qwen respectively, of which all but seven were
duplicate edges; out-of-range vertices and self-loops together account
for the remaining seven. The parser therefore removes almost no edges
beyond repeats, so the hallucination counts reflect model output
rather than post-processing. 

\subsection{Measurement}

For each graph we form the combinatorial Laplacian and compute its
full sorted spectrum. The distortion $\EMD$ is evaluated as the
Wasserstein distance between the two spectra viewed as equally
weighted samples of size $n$, which coincides with \eqref{eq:emd}. In
$\ACR$ we treat $\lambda'_2$ below $10^{-9}$ as exactly zero, so a
disconnected reconstruction reports $\ACR=0$ rather than a numerical
artifact.

From the retained edge lists we obtain $a=|E'\setminus E|$ and
$d=|E\setminus E'|$ directly, and hence the edit class of
Definition~\ref{def:classes}. The certificate of
Corollary~\ref{cor:cert} is evaluated as $\lceil nX/4\rceil$ and
compared against the true $\min(a,d)$. Since $E$ and $\Ep$ are both
known, one-sided reconstructions are further separated into those
returning exactly the prompted sublist, $E'=\Ep$, and those returning
a proper subset of it.

Completion behavior is measured relative to $\Ep$ and $\Ew$ by
\begin{align*}
\text{kept}&=|E'\cap\Ep|, &
\text{dropped}&=|\Ep\setminus E'|,\\
\text{recovered}&=|E'\cap\Ew|, &
\text{hallucinated}&=|E'\setminus E|,
\end{align*}
so that $a$ equals the number of hallucinated edges and $d$ equals
dropped plus $|\Ew|$ minus recovered. These identities were verified
on all 135 reconstructions.

\subsection{A null model for recovery}

A model that emits non-prompt edges without using structure would
still recover some withheld edges by chance, at a rate that grows with
the number of edges it emits. To separate signal from volume we
condition on that number. Let $r=\text{recovered}+\text{hallucinated}$
be the count of emitted edges outside $\Ep$, and let
$N=\binom{n}{2}-|\Ep|$ be the number of vertex pairs available to a
structure-blind model. Drawing $r$ of those $N$ pairs uniformly makes
the recovered count hypergeometric with mean $rq$ and variance
$rq(1-q)(N-r)/(N-1)$, where $q=|\Ew|/N$. Summing these moments over
instances gives the expectation and variance of total recovery under
the null, and we report the standardized deviation $z$ of the observed
total. The test asks whether a model's hits exceed chance given how
often it guessed, not whether it guessed often.

\subsection{Link prediction baselines}

To calibrate recovery against structure-aware heuristics, we compare
each model to three classical link predictors at the same emission
volume. For a model and instance, let $r$ be the number of edges the
model emitted outside $\Ep$. We rank all candidate pairs in
$\binom{V}{2}\setminus\Ep$ on the visible graph $(V,\Ep)$ by common
neighbors,
\begin{equation*}
s_{\mathrm{CN}}(u,v)=|\Gamma_{\mathrm p}(u)\cap\Gamma_{\mathrm p}(v)|,
\end{equation*}
Adamic--Adar \cite{AdamicAdar:2003},
\begin{equation*}
s_{\mathrm{AA}}(u,v)=\!\!\!\sum_{w\in\Gamma_{\mathrm p}(u)\cap
\Gamma_{\mathrm p}(v)}\!\!\!\frac{1}{\log\deg_{\mathrm p}(w)},
\end{equation*}
and preferential attachment \cite{LibenNowell:2007},
\begin{equation*}
s_{\mathrm{PA}}(u,v)=\deg_{\mathrm p}(u)\,\deg_{\mathrm p}(v),
\end{equation*}
where $\Gamma_{\mathrm p}$ and $\deg_{\mathrm p}$ are neighborhoods
and degrees in $(V,\Ep)$. Each baseline selects the top $r$ pairs.
Score ties are broken uniformly at random and the recovered count is
averaged over $1{,}000$ tie permutations, which matters on the grid
family where many pairs share a zero score. The random baseline is the
hypergeometric expectation $r\,|\Ew|/(\binom{n}{2}-|\Ep|)$ at the same
budget.

\section{Results}
\label{sec:results}

\subsection{Bounds and attainment}

Numerical evaluation agrees with \eqref{eq:sandwich} to tolerance
$10^{-6}$ on all 135 outputs. The empirical question is therefore
which equality regime each output occupies.
Table~\ref{tab:census} shows 77 one-sided outputs, all attaining both
bounds as required by Corollary~\ref{cor:collapse}: 57 exact sublist
returns and 20 proper subsets. For these $\EMD$ is a rescaled edge
count carrying no information about which edges changed. No output is a
pure addition. The remaining 58 are mixed; 29 attain the lower bound
and none the upper. For 57 of the 58, an added edge meets a deleted
edge, so Proposition~\ref{prop:upper} rules out upper bound
attainment. The exception is Llama on
grid g35, with $a=2$, $d=15$, $\EMD=0.867$ and upper bound $1.133$:
vertex disjoint yet below the upper bound, so the condition is
necessary but not sufficient. The largest dominant edit is Mistral on
g39, with $a=192$, $d=4$, $\ECR=5.48$ and $X=0$.

\begin{table}[t]
\caption{Attainment of \eqref{eq:sandwich} by edit class at tolerance
$10^{-6}$. The first two rows are one-sided; the last two are mixed.}
\label{tab:census}
\centering
\begin{tabular}{lrrr}
\toprule
Edit class & $N$ & Lower attained & Upper attained\\
\midrule
Sublist returned exactly & 57 & 57 & 57\\
Other one-sided & 20 & 20 & 20\\
Mixed, dominant & 29 & 29 & 0\\
Mixed, certified & 29 & 0 & 0\\
\midrule
Total & 135 & 106 & 77\\
\bottomrule
\end{tabular}
\end{table}

\subsection{Certificates and crossings}

The trace excess of \eqref{eq:excess} is positive on 29
reconstructions. All 29 are mixed, as guaranteed by
Corollary~\ref{cor:cert}; the stored edge lists provide an
implementation check, with $\lceil nX/4\rceil$ at most the true
$\min(a,d)$ on every one. The estimate is uniformly conservative: it
returns 1 on 24 instances, 2 on two, 3 on two and 5 on one, while the
true overlap ranges from 2 to 24.

The set of instances with $X>0$ coincides exactly with the set whose
sorted spectral curves cross, and no one-sided instance exhibits a
crossing, as required by Proposition~\ref{prop:crossing}(iii).
Figure~\ref{fig:portraits} shows this on three instances: the two
curves that cross are the two certified reconstructions in the panel
set, and the remaining seven track the original without a sign change.
The coarse test of Proposition~\ref{prop:acr}(ii) fires on 10 of the
29 and never outside the certified class, and $\ACR=0$ agrees with
disconnection of $G'$ with no mismatch across the corpus, an
implementation check for Proposition~\ref{prop:acr}(iii). Disconnection is common: 41 of the
135 reconstructions have $\ACR=0$, concentrated in the scale-free and
grid families with 24 and 14 instances against 3 for community
graphs. Preferential attachment and lattice graphs contain many
degree two vertices, which a withheld quarter of the edges can
isolate, whereas dense blocks are redundant enough to survive. The
counts also invert the ordering of edit volume: Qwen disconnects 19
reconstructions and Mistral only 9, consistent with the fact that
deletion can isolate vertices whereas additional edges can reconnect
components.

Two instances show what the certificate buys. On community instance
g5 the Llama reconstruction has $\ECR=1$, $\ACR=0.84$ and
$\EMD=0.16$. None of these scalars identifies the simultaneous 19
additions and 19 deletions; their joint residual $X=0.161$ does.
Scale-free instance g22 similarly preserves edge count while adding 15
edges and deleting 15, with the largest excess in the corpus at
$X=0.554$.

\begin{figure*}[t]
\centering
\includegraphics[width=\textwidth]{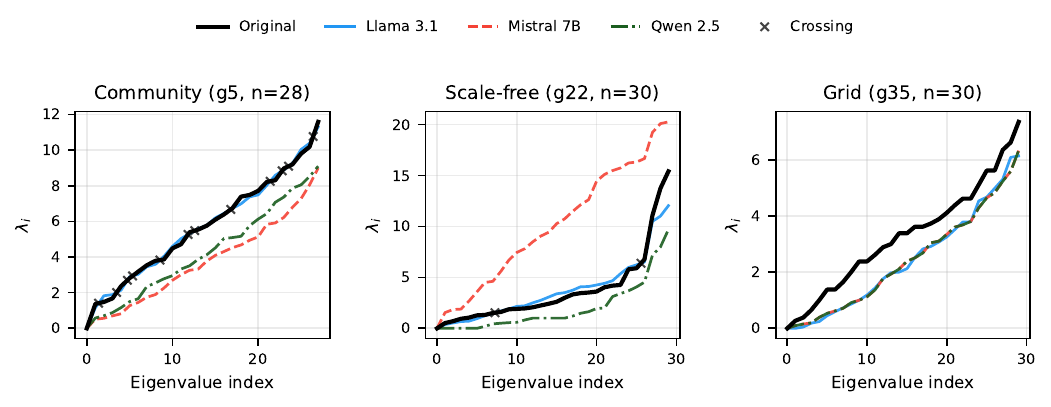}
\caption{Sorted Laplacian spectra for one medium instance per family;
crossings mark $X>0$ (Proposition~\ref{prop:crossing}). The two
crossing curves are the certified reconstructions of g5 and g22. In
g35 the Mistral and Qwen curves coincide (identical edge lists) while
the Llama curve is mixed yet spectrally dominated, the one corpus
instance whose added and deleted edges are vertex disjoint.}
\label{fig:portraits}
\end{figure*}

\subsection{Edit policies differ across models}

The three models occupy different regions of the classification of
Definition~\ref{def:classes}. Aggregate distortion orders them by
magnitude, as Figure~\ref{fig:atlas} shows, but it obscures a
difference in editing policy that is visible in
Table~\ref{tab:policy} and resolved by family in
Table~\ref{tab:composition}.

Qwen exhibits a predominantly copy-oriented output pattern. It is one
sided on 44 of 45 instances, returning the prompted sublist verbatim
36 times, and emits non-prompt edges on one instance, which is
certified mixed. Llama produces mixed edits on 31 of 45 instances and
is the source of 19 of the 29 certificates. Mistral is bimodal: it is
one-sided on 19 instances and produces the largest hallucination
volume in the corpus on the remainder, reaching $\ECR=5.48$ on g39 and
$\ECR=3.65$ on g16.

Recovery of withheld edges is low for all three. Out of 805 withheld
edges each, Llama recovers 34, Mistral 46 and Qwen 1. Under the null
model of Section~\ref{sec:setup}, which conditions on how many
non-prompt edges each model emitted, Llama is above chance at
$z=+5.46$, whereas Mistral is indistinguishable from structure-blind
guessing at $z=+1.17$ despite its higher raw count: its 46 hits follow
from 732 attempts against a chance expectation of 40.6. Qwen emits
only four non-prompt edges in total, so we do not interpret a
standardized recovery score.

Chance is a weak reference point, so Table~\ref{tab:baselines}
compares each model against classical link prediction heuristics given
the same emission budget on the same visible graph. In aggregate Llama's 34 exceeds the best baseline, Adamic--Adar at
27.2, while Mistral's 46 falls below both common neighbors at 47.4
and Adamic--Adar at 50.9. The per family breakdown shows that this ordering
is not uniform. On community graphs Adamic--Adar recovers 17.0 and 20.4 against the
models' 11 and 12, since dense blocks make most withheld edges close a
triangle. On scale-free graphs 68 per cent of withheld edges have no
common neighbors, because preferential attachment graphs are locally
tree like, which is why the neighborhood scores are weak there while
preferential attachment, whose score is a degree product and needs no
triangles, recovers 14.8 and 33.6 against the models' 9 and 21. In both families the models are beaten by the heuristic that
matches the family.

The grid family is the only one on which both evaluated models
outperform all matched-volume baselines. Here common neighbors and
Adamic--Adar have no signal: every withheld grid edge has zero common
neighbors because the lattice is bipartite, so both rank the withheld
edges at the bottom of their candidate lists. Preferential attachment
remains defined but is also weak. The models recover 14 and 13 there.
Since the fixed row-major labeling exposes adjacency through index
differences of one or the row length, the model advantage is
consistent with serialization cues rather than topological reasoning,
and is specific to the labeling and prompt format of
Section~\ref{sec:setup}.

Two further patterns hold across the corpus. No reconstruction
recovers a withheld edge without also hallucinating at least one edge,
so no completion is clean. Reproduction itself is lossy: of the 2364
prompt edges supplied to each model, Llama fails to return 54, Mistral
109 and Qwen 40, or 2.3, 4.6 and 1.7 per cent.

\begin{table}[t]
\caption{Edit policy per model. Classes are exact sublist / other one
sided / mixed; recovery is out of 805 withheld edges; $z$ is the
standardized recovery under the null model of Section~\ref{sec:setup}.
Qwen's score is not reported because it emitted only four non-prompt
edges.}
\label{tab:policy}
\centering
\begin{tabular}{lcrrrr}
\toprule
Model & Classes & Recov. & Halluc. & Dropped & $z$\\
\midrule
Llama 3.1 & 9/5/31 & 34 & 430 & 54 & $+5.46$\\
Mistral 7B & 12/7/26 & 46 & 686 & 109 & $+1.17$\\
Qwen 2.5 & 36/8/1 & 1 & 3 & 40 & ---\\
\bottomrule
\end{tabular}
\end{table}

\begin{table}[t]
\caption{Edit-class composition by model and graph family, as exact
sublist / other one-sided / mixed.}
\label{tab:composition}
\centering
\begin{tabular}{lccc}
\toprule
Model & Community & Scale-free & Grid\\
\midrule
Llama 3.1 & 6/0/9 & 1/3/11 & 2/2/11\\
Mistral 7B & 2/5/8 & 4/0/11 & 6/2/7\\
Qwen 2.5 & 11/3/1 & 13/2/0 & 12/3/0\\
\bottomrule
\end{tabular}
\end{table}

\begin{table*}[t]
\caption{Withheld-edge recovery at matched emission volume, averaged
over $1{,}000$ tie permutations. CN: common neighbors; AA:
Adamic--Adar; PA: preferential attachment; Rand: hypergeometric
expectation. Bold marks the largest recovery in each row. Qwen is
omitted (four non-prompt edges total).}
\label{tab:baselines}
\centering
\footnotesize
\begin{tabular}{llrrrrrr}
\toprule
Model & Family & Wthld & LLM & Rand & CN & AA & PA\\
\midrule
Llama 3.1 & Community & 386 & 11 & 4.2 & \textbf{17.1} & 17.0 & 1.8\\
 & Scale-free & 225 & 9 & 6.9 & 7.9 & 10.1 & \textbf{14.8}\\
 & Grid & 194 & \textbf{14} & 3.2 & 0.1 & 0.1 & 0.2\\
\midrule
Mistral 7B & Community & 386 & 12 & 8.1 & 19.4 & \textbf{20.4} & 6.7\\
 & Scale-free & 225 & 21 & 21.4 & 21.8 & 24.3 & \textbf{33.6}\\
 & Grid & 194 & \textbf{13} & 11.1 & 6.2 & 6.2 & 4.1\\
\bottomrule
\end{tabular}
\end{table*}

\subsection{The degenerate regime and aggregate distortion}

On all 57 instances that return the prompted sublist, the closed form
\eqref{eq:echoform} holds to machine precision, with a maximum
deviation of $1.8\times10^{-15}$ computed from the raw spectra.
The closed form matches the observed mean in all nine cells. The size trend follows
the average degree as predicted: community distortion grows from 0.911
to 2.087 across the three levels, while grid and scale-free
distortion stay near constant, at 0.751 to 0.866 and 0.933 to 0.960
respectively. These values themselves do not distinguish models.

Figure~\ref{fig:atlas} shows the mean distortion per model and family,
which is how such evaluations are usually summarized. Two readings of
that matrix are misleading. First, the Qwen row of 1.457, 1.004 and
0.830 is produced by 36 verbatim returns out of 45, so by
Proposition~\ref{prop:echo} it is close to a deterministic function of
the benchmark densities rather than a profile of the model. Second,
the Mistral grid cell has mean 1.713 and maximum 13.93, the latter
from g39 alone, so the cell average describes no typical instance.
The matrix supports the conclusion that Mistral distorts most, which
is true, and supports no conclusion at all about how.

\begin{figure}[t]
\centering
\includegraphics[width=\columnwidth]{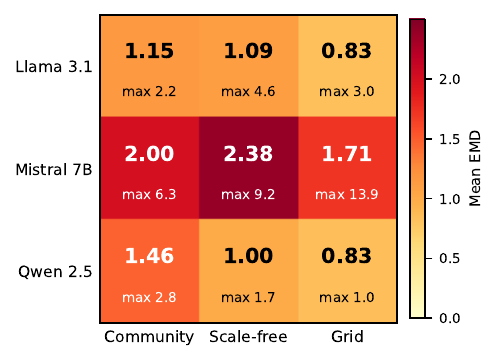}
\caption{Mean $\EMD$ and cell maximum over 15 outputs per
model-family cell. The Qwen row largely tracks the benchmark density
profile; the Mistral grid cell is driven by one instance.}
\label{fig:atlas}
\end{figure}

\section{Discussion and Conclusion}
\label{sec:discussion}

The bounds are independent of how $G'$ is produced. In our corpus 77
of 135 outputs are one-sided, so by Corollary~\ref{cor:collapse} the
distortion is a rescaled edge count with no edge-identity information:
on the 57 exact sublist returns it is fixed by $n$, $|E|$ and $\rho$,
and on the other 20 it reflects output volume but not edit mechanism.

The trace excess recovers part of this loss at no additional
measurement cost. It certifies 29 mixed outputs, including g5 and g22
where $\ECR=1$ hides 38 and 30 edits. The certificate is intentionally
conservative: 29 other mixed outputs have $X=0$ because their sorted
spectra remain pointwise ordered. Example~\ref{ex:k4} gives one
sufficient positive semidefinite mechanism, and Example~\ref{ex:iso}
shows that mixed edits can even be Laplacian isospectral. Thus $X>0$
is conclusive, whereas $X=0$ is not. Three counts are zero across the
corpus: no output adds edges without also losing some, none returns a
rearranged same-sized subset, and none recovers a withheld edge
without hallucinating one, so no completion is clean.

The models exhibit distinct policies. Qwen mostly copies the prompt,
Llama more often attempts completion, and Mistral is bimodal and can
be hallucination-heavy. Recovery stays low, and the matched-volume
baselines show that raw hits must be read relative to emission volume
and graph family. Aggregate $\EMD$ orders the models by magnitude,
as Figure~\ref{fig:atlas} shows, but does not reveal these mechanisms.
For fixed-vertex evaluations based on combinatorial-Laplacian $W_1$,
EMD should be reported with ECR and $X$: ECR records net edge-count
change, while $X>0$ certifies simultaneous addition and deletion.

\paragraph{Limitations}
We study three 7--8B open-weight models, synthetic graphs, one keep
ratio, and one prompt, node labeling and edge ordering. Larger or
proprietary models may behave differently, and random relabeling or a
different serialization may change the observed policies. The withheld
set is generally not uniquely identifiable, so the experiment
characterizes behavior under this protocol rather than general
graph-reasoning ability. Parsing fixes the vertex set, and the uniform
null model does not represent all plausible structured guessing.

\paragraph{Conclusion}
We proved a sharp sandwich for combinatorial-Laplacian $W_1$
distortion, identified its exact one-sided collapse, and derived a
scalar certificate of mixed editing with $\min(a,d)\ge nX/4$. Across
135 model outputs the degenerate regime is the majority and the
certificate detects mixed edits hidden by edge-count summaries. The
conclusions apply to fixed-vertex evaluations using this spectral
metric; $X>0$ is conclusive, while $X=0$ remains compatible with mixed
editing.

\end{document}